\documentclass{article}

\usepackage[nonatbib,final]{neurips_2020}

\usepackage{multicol}
\usepackage{caption}
\usepackage[utf8]{inputenc} % allow utf-8 input
\usepackage[T1]{fontenc}    % use 8-bit T1 fonts
\usepackage{url}            % simple URL typesetting
\usepackage{booktabs}       % professional-quality tables
\usepackage{graphicx}
\usepackage{amsfonts}       % blackboard math symbols
\usepackage{nicefrac}       % compact symbols for 1/2, etc.
\usepackage{microtype}      % microtypography
\usepackage{mathtools}
\usepackage{enumitem}
\usepackage{amsmath,amssymb,amsthm,relsize}
\newtheorem{theorem}{Theorem}
\newtheorem{proposition}{Proposition}
\newcommand{\allfor}{\displaystyle\mathop{\mathlarger{\forall}}}

\newtheorem{lemma}[theorem]{Lemma}

\title{A New Neural Network Architecture Invariant \\ to the Action of Symmetry Subgroups}

\author{%
  Piotr Kicki \mbox{~~~} Piotr Skrzypczyński \\
  Institute of Robotics and Machine Intelligence\\
  Poznan University of Technology, Poznań, Poland \\
  \texttt{\{piotr.kicki,piotr.skrzypczynski\}@put.poznan.pl} \\
  \And
  Mete Ozay\\
  \texttt{meteozay@gmail.com}\\
}

\begin{document}

\maketitle

\begin{abstract}
We propose a  computationally efficient $G$-invariant
neural network that approximates functions invariant to the action of a given
permutation subgroup $G \leq S_n$ of the symmetric group on input data.
The key element of the proposed network architecture is a new $G$-invariant transformation module, which produces a $G$-invariant latent representation of the input data.
Theoretical considerations are supported by numerical experiments,
which demonstrate the effectiveness and strong generalization properties of the proposed method in comparison to other $G$-invariant neural networks.
\end{abstract}

\section{Introduction}
The design of probabilistic models which reflect symmetries existing in data is considered an important task following the notable success of deep neural networks, such as convolutional neural networks (CNNs) \cite{cnn2} and PointNet \cite{pointnet}.
Models with superior performance can be obtained using prior knowledge about data and desired properties of the models, such as permutation invariance \cite{pointnet}. Similarly, translation equivariance can be exploited for CNNs \cite{G_equiv_cnns} to reduce their model size.

Nevertheless, researchers have been working on developing a general approach which enables to design architectures that are invariant and equivariant to the actions of particular groups. Invariance and equivariance of learning models to actions of various groups have been discussed in the literature \cite{cohen2,itpooling,parameter_sharing,deep_sets}. 
However, in this paper, we only consider invariance to permutation groups $G$, which are the subgroups of the symmetric group $S_n$ of all permutations on a finite set of $n$ elements, as it covers many interesting applications.

\begin{figure}
\centering
\begin{minipage}[t]{.47\textwidth}
  \centering
  \includegraphics[width=\linewidth]{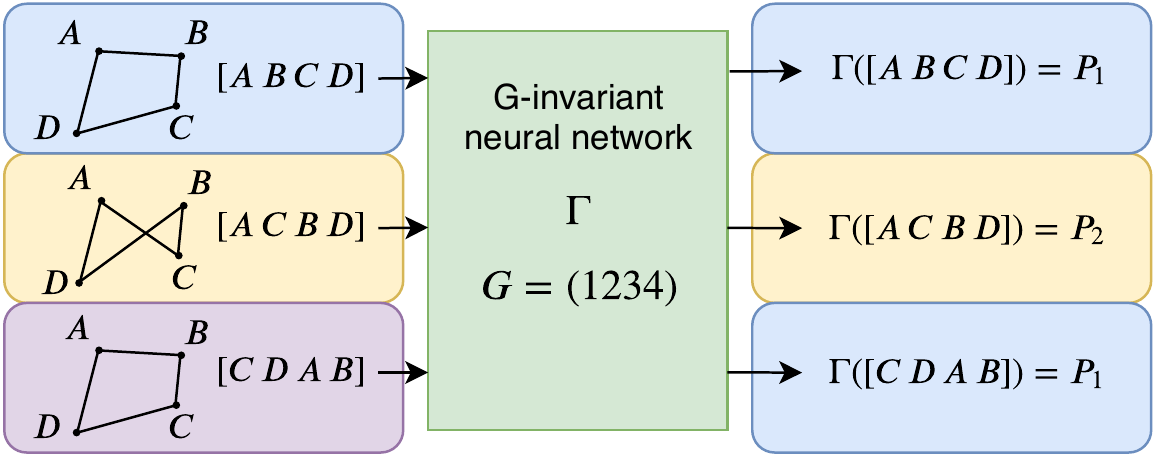}
  \captionof{figure}{An illustration of employment of the proposed {$G$-invariant} neural network $\Gamma$ for estimation of area of quadrangles. If the consecutive vertices are provided in the same order (e.g. $[C\,D\,B\,A]$ or $[A\,B\,C\,D]$), then the network $\Gamma$ computes the same area $P_1$. However, if the order changes (i.e.  $[A\,C\,B\,D]$), then the network $\Gamma$ estimates a different area $P_2$.}
  \label{fig:intro}
\end{minipage}
\hspace{0.7pc}%
\begin{minipage}[t]{.5\textwidth}
  \centering
  \includegraphics[width=\linewidth]{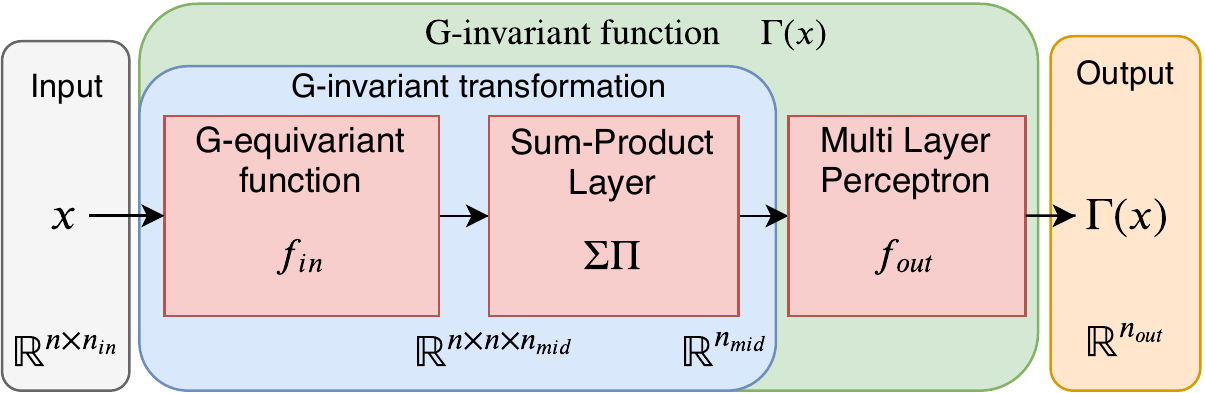}
  \captionof{figure}{An illustration of the proposed {$G$-invariant neural network}. An input $x$ is processed by the $G$-invariant transformation (blue), which produces a $G$-invariant representation of the input. Then, the $G$-invariant representation is passed to the Multi Layer Perceptron which produces the output vector $\Gamma(x)$.}
  \label{fig:ginv}
\end{minipage}
\vspace{-0.35cm}
\end{figure}

An example of the employment of the proposed $G$-invariant network for a set of quadrangles is illustrated in Figure~\ref{fig:intro}. The network $\Gamma$ receives a matrix representation of the quadrangles (i.e. a vector of 4 points on a plane) and outputs the areas covered by those quadrangles. One can spot that, no matter which point will be given first, if the consecutive vertices are provided in the right order, then the area of the figure will remain the same. Such a property can be described by $G$-invariance, where $G = (1234)$\footnote{$G=(1234)$ denotes a group $G$ generated by the permutation $(1234)$, in which the first element is replaced by the second, the second by the third and so on, till the last element being replaced by the first one.}.

Recently, Maron et al. \cite{maron19} proposed a $G$-invariant neural network architecture for some finite subgroups $G \leq S_n$, and proved its universality. Unfortunately, their solution is intractable for larger inputs and groups, because of the rapidly growing size of tensors and the number of operations needed for forward and backward passes in the network.

The aim of this paper is to propose a method that enables us to design a novel $G$-invariant architecture for a given finite group $G \leq S_n$, which is tractable and generalizes well.

\section{Our Proposed G-invariant Neural Network Architecture}
\vspace{-2mm}
We introduce a novel $G$-invariant neural network architecture, which exploits the theory of invariant polynomials to achieve a flexible scheme for $G$-invariant transformation of data for some known and finite group $G \leq S_{n}$, where $S_n$ is a symmetric group and $|G| = m$.
We assume that an input $x \in \mathbb{R}^{n \times n_{in}}$ to the proposed network is a tensor\footnote{We use matrix notation to denote tensors in this paper.} $x = [x_1\, x_2\,  \dots\, x_n]^T$ of $n$ vectors $x_i \in \mathbb{R}^{n_{in}}$, ${i =1,2,\ldots, n}$. 
A function $f: \mathbb{R}^{n \times n_{in}} \rightarrow \mathbb{R}$  is $G$-invariant if $f$ satisfies\footnote{$\allfor_{y \in Y} P(Y)$ means that ``predicate $P(Y)$ is true for all $y \in Y$''.}
\begin{equation}
\label{eq:Ginv}
  \allfor_{x \in \mathbb{R}^{n \times n_{in}}}\allfor_{g \in G}\,f\left(g(x)\right) = f(x), 
\end{equation}
where the action of the group element $g$ on $x$ is defined by
\begin{equation}
\label{eq:Gequiv}
    g(x) = \{x_{\sigma_g(1)}, x_{\sigma_g(2)}, \dots, x_{\sigma_g(n)}\},
\end{equation}
where $\sigma_g(i)$ denotes the action of the group element $g$ on the specific index $i$ and $x_{\sigma_g(i)} \in \mathbb{R}^{n_{in}}$.

Our proposed $G$-invariant neural network is illustrated in Figure \ref{fig:ginv}, and defined by a  function ${\Gamma: \mathbb{R}^{n \times n_{in}} \rightarrow \mathbb{R}^{n_{out}}}$ of the following form
\begin{equation}
\label{eq:Gamma}
   \Gamma(x) = f_{out}(\Sigma\Pi(f_{in}(x))),
\end{equation}
where $f_{in}$ is a $G$-equivariant input transformation function, $\Sigma\Pi$ is a function which comprises {$G$-invariant} transformation when combined with $f_{in}$, and $f_{out}$ is an output transformation function.
The general idea of the proposed architecture is to define a $G$-invariant transformation, which uses the sum of $G$-invariant polynomials ($\Sigma\Pi$) of $n$ variables, which are the outputs of $f_{in}$. This transformation produces a $G$-invariant feature vector, which is processed by another function $f_{out}$ that is approximated by a Multi-Layer Perceptron.

First, let us define the $G$-equivariant input transformation function $f_{in}: \mathbb{R}^{n \times n_{in}} \rightarrow \mathbb{R}^{n \times n \times n_{mid}}$, where $n_{mid}$ is the size of the feature vector. This function can be represented as a vector ${\Phi = [\phi_1\, \phi_2\, \dots\, \phi_n]}$ of neural networks, where each function $\phi_i: \mathbb{R}^{n_{in}} \rightarrow \mathbb{R}^{n_{mid}}$ is applied on all elements of the set of input vectors $\{x_i\}_{i=1}^n$, and transforms them to the $n_{mid}$ dimensional vector.
As a result, the operation of the $f_{in}$ function can be formulated by
\begin{equation}
\label{eq:fin}
    f_{in}(x) = 
    \begin{bmatrix}
        \Phi(x_1)\\
        \Phi(x_2)\\
        \vdots\\
        \Phi(x_n)\\
    \end{bmatrix} = 
    \begin{bmatrix}
        \phi_1(x_1) & \ldots & \phi_n(x_1)\\
        \vdots & \ddots & \vdots \\
        \phi_1(x_n) & \ldots & \phi_n(x_n)\\
    \end{bmatrix}.
\end{equation}
One can see that $f_{in}(x)$ is $G$-equivariant, since the action of the vector $\Phi$ of functions is the same for each element of the vector $x$, thus it transposes the rows of the matrix form \eqref{eq:fin} according to $g \in G$, which is equivalent to transposing the rows after the calculation of $f_{in}(x)$.

Second, we define the function $\Sigma\Pi: \mathbb{R}^{n \times n \times n_{mid}} \rightarrow \mathbb{R}^{n_{mid}}$, which constructs $G$-invariant polynomials of outputs obtained from $f_{in}$, by
\vspace{-0.09in}
\begin{equation}
\label{eq:sigmapi}
    \Sigma\Pi(x) = \sum_{g \in G}\prod_{j=1}^{n} x_{\sigma_g(j), j}.
\end{equation}
To see the $G$-invariance of $\Sigma\Pi(f_{in}(x))$, we substitute $x$ from \eqref{eq:sigmapi} with \eqref{eq:fin} to obtain
\vspace{-0.1in}
\begin{equation}
\label{eq:sigmapifin_G}
    \Sigma\Pi(f_{in}(x)) = \sum_{g \in G}\prod_{j=1}^{n}  \phi_j(x_{\sigma_g(j)}).
\end{equation}
Then, we can show that \eqref{eq:sigmapifin_G} is $G$-invariant by checking whether \eqref{eq:Ginv} holds for any input $x$ and any group element $g' \in G$ as follows:
\vspace{-0.05in}
\small
\begin{equation}
\label{eq:anyginG}
% \begin{split}
    \Sigma\Pi(g'(f_{in}(x))) = \sum_{g \in G}\prod_{j=1}^{n} \phi_j(x_{\sigma_{g'}(\sigma_g(j))}) 
    = \sum_{g \in G}\prod_{j=1}^{n}  \phi_j(x_{\sigma_g(j)}) = \Sigma\Pi(f_{in}(x))
% \end{split},
\end{equation}
\normalsize
since any group element acting on the group leads to the group itself.
The product operation employed in \eqref{eq:sigmapi}-\eqref{eq:anyginG} is performed element-wise.
We define the output function $f_{out}: \mathbb{R}^{n_{mid}} \rightarrow \mathbb{R}^{n_{out}}$ following the structure of a typical fully connected neural network by 
\vspace{-0.05in}
\begin{equation}
\label{eq:fout}
  f_{out}(x) = \sum_{i=1}^{N}c_i \sigma\left(\sum_{j=1}^{n_{mid}} w_{ij} x_j + h_i\right),
\end{equation}
where $N \in \mathbb{N}_{+}$ is a parameter, $\sigma$ is a non-polynomial activation function and $c_i, w_{ij}, h_i \in \mathbb{R}$ are coefficients. We elucidate the universality of our $G$-Invariant network in the next proposition\footnote{The proof of the proposition is given in the supplemental material.}.
\begin{proposition}
\label{thm:proposition}
The network function (\ref{eq:Gamma}), can approximate any $G$-invariant function $f: V \rightarrow \mathbb{R}$, where $V$ is a compact subset of $\mathbb{R}^{n \times n_{in}}$ and $G \leq S_n$ is a finite group, as long as number of features $n_{mid}$ obtained at the output of input transformation network $f_{in}$ is greater than or equal to the size $N_{inv}$ of the generating set $\mathcal{F}$ of polynomial $G$-invariants.
\end{proposition}

\vspace{-0.2cm}
\section{Experimental Analysis}
\vspace{-0.3cm}
We evaluate the accuracy of the proposed architecture and analyze its invariance properties in the following two tasks\footnote{The code is available on \url{https://github.com/Kicajowyfreestyle/G-invariant}}:

\textbf{(i) $G$-invariant Polynomial Regression}: The goal of this task is to train a model to approximate a $G$-invariant polynomial. In the experiments, we consider various polynomials: $P_{\mathbb{Z}_k}, P_{S_k}, P_{D_{2k}}, P_{A_k}$ and $P_{S_k \times S_l}$, which are invariant to the cyclic group $\mathbb{Z}_k$, permutation group $S_k$, dihedral group $D_{2k}$, alternating group $A_k$ and direct product of two permutation groups $S_k \times S_l$, respectively. To examine generalization abilities of models of the proposed $G$-invariant network architecture, they were trained using only 16 different random points in $[0; 1]^5$, whereas 480 and 4800 randomly generated points were used for validation and testing, respectively.

\textbf{(ii) Estimation of the Area of Convex Quadrangles:} In this task, models are trained to estimate areas of convex quadrangles. An input is a vector of 4 points lying in $\mathbb{R}^{4 \times 2}$, each described by its $x$ and $y$ coordinates. The desired estimator is a simple example of the $G$-invariant function, where $G$ is identified by $\mathbb{Z}_4 = (1234)$.
Both training and validation set contains 256 examples (randomly generated convex quadrangles with their areas), while the test dataset contains 1024 examples.

\subsection{Compared Architectures and Models}
\vspace{-1mm}
All of the experiments presented below consider networks of different architectures for which the number of weights was fixed at a similar level for the given task for a fair comparison.
The employed architectures are the following:
\vspace{-0.1in}
\begin{itemize}[leftmargin=2em]
\setlength\itemsep{0.0em}
    \item FC $G$-avg: Fully connected (FC) neural network with group averaging,
    \item Conv1D $G$-avg: 1D convolutional neural network with group averaging,
    \item FC $G$-inv: $G$-invariant neural network (\ref{eq:Gamma}) implementing $f_{in}$ using a FC neural network,
    \item Conv1D $G$-inv: $G$-invariant neural network (\ref{eq:Gamma}) implementing $f_{in}$ using 1D CNN,
    \item Maron: $G$-invariant network proposed in \cite{maron19}.
\end{itemize}

\begin{table*}[ht]
\vskip -0.1in
\caption{Mean absolute errors (MAEs) [$10^{-2}$] of the compared $G$-invariant models for the task of $G$-invariant polynomial regression.}
\label{tab:poly}
\begin{center}
\begin{small}
\begin{sc}
\begin{tabular}{lcccc}
\toprule
Network & Train & Validation & Test & \#Weights [$10^3$]\\
\midrule
FC $G$-avg              & 15.15 $\pm$ 5.49  & 16.48 $\pm$ 0.73  & 16.89 $\pm$ 0.76   & 24.0\\
\textbf{$G$-inv (ours) }         & 2.65 $\pm$ 0.91   & 7.32 $\pm$ 0.55   & 7.46  $\pm$ 0.56   & 24.0\\
Conv1D $G$-avg & 8.98 $\pm$ 6.39 & 11.43 $\pm$ 4.29 & 11.78 $\pm$ 4.79 & 24.0\\
\textbf{Conv1D $G$-inv (ours) }  & \textbf{0.87 $\pm$ 0.12} & \textbf{2.57 $\pm$ 0.37} & \textbf{2.6 $\pm$ 0.4} & 24.0\\
Maron                   & 2.41 $\pm$ 0.82   & 5.74 $\pm$ 1.19   & 5.93  $\pm$ 1.18   & 24.2\\
\bottomrule
\end{tabular}
\end{sc}
\end{small}
\end{center}
\vskip -0.2in
\end{table*}

\subsection{Results for $\mathbb{Z}_5$-invariant Polynomial Regression}
\vspace{-1mm}
\label{sec:poly}
The accuracy of the examined models is given in Table \ref{tab:poly}.
We observe that our proposed Conv1D $G$-inv outperforms all of the other architectures on both datasets. Both Maron and FC $G$-inv obtain worse MAE, but they significantly outperform the Conv1D $G$-avg and FC $G$-avg.
Moreover, those architectures obtain large standard deviations for the training dataset, because sometimes they converge to different error values. In contrast, the performance of the $G$-inv based models and the Maron model is relatively stable under different weight initialization.
While the results are similar for our architecture and the approach introduced in \cite{maron19}, the number of computations used by the Maron's model is significantly larger\footnote{See supplementary material for more detailed analysis.}. The inference time for both networks differs notably, and equals $2.3 \pm 0.4$ms for Conv1D $G$-inv and $21.4 \pm 1.5$ms for Maron.

\vspace{-2mm}
\begin{table*}[ht]
\caption{Mean absolute errors (MAEs) [$10^{-3} \text{unit}^2$] of the compared $G$-invariant models for the task of convex quadrangle area estimation.}
\vspace{-2mm}
\label{tab:area}
%\vskip 0.15in
\begin{center}
\begin{small}
\begin{sc}
\begin{tabular}{lcccc}
\toprule
Network & Train & Validation & Test & \#Weights\\
\midrule
FC $G$-avg              & 7.0 $\pm$ 0.6 & 9.6 $\pm$ 1.0 & 9.4 $\pm$ 0.9   & 1765\\
\textbf{$G$-inv (ours)} & 7.4 $\pm$ 0.4 & 8.0 $\pm$ 0.3 & 8.3 $\pm$ 0.5   & 1785\\
Conv1D $G$-avg          & 16.9 $\pm$ 7.7 & 16.8 $\pm$ 5.3 & 18.5 $\pm$ 6.8 & 1667\\
\textbf{Conv1D $G$-inv (ours)}  & \textbf{6.0 $\pm$ 0.3} & \textbf{7.3 $\pm$ 0.3} & \textbf{7.5 $\pm$ 0.5}   & 1673\\
Maron                   & 13.9 $\pm$ 0.9 & 22.3 $\pm$ 1.2 & 23.4 $\pm$ 1.3   & 1802\\
\bottomrule
\end{tabular}
\end{sc}
\end{small}
\end{center}
\vskip -0.15in
\end{table*}

\subsection{Results for Estimation of Areas of Convex Quadrangles}
\vspace{-1mm}
The accuracy of the examined models is reported in Table \ref{tab:area}.
The results show that the model utilizing the approach presented in this paper obtains the best performance on all three datasets. 
Furthermore, it generalizes much better to the validation and test dataset than any other tested approach.
We observe that, besides the proposed $G$-invariant architecture, the only approach which was able to reach a low level of MAE in the polynomial approximation task (Maron) is unable to accurately estimate the area of the convex quadrangle.

\section{Conclusion}
\vspace{-3mm}
In this paper, we have proposed a novel $G$-invariant neural network architecture that uses two standard neural networks, connected with the proposed Sum-Product Layer denoted by $\Sigma\Pi$.
We conducted two experiments to analyze the accuracy of the proposed $G$-invariant architecture in comparison with the other $G$-invariant architectures proposed in the literature. The results demonstrate that the proposed $G$-invariant neural network outperforms all other approaches in both tasks. 
We believe that the proposed $G$-invariant neural networks can be employed by researchers to learn group invariant models efficiently in various applications in machine learning, computer vision and robotics. In the future work, we plan to apply the proposed networks for various tasks in computer vision and robot learning, which require vector map processing using the geometric structure of data.

\acksection
\vspace{-3mm} % they have special section for Ack., unfortunately makes paper too long.
This research was partially supported by TAILOR, a project funded by EU Horizon 2020 research and innovation program under GA No. 952215.

\bibliography{example_paper}
\bibliographystyle{plain}

\newpage
\section{Supplementary Material}

For consistency, we first provide our proposed G-invariant neural network architecture in the next section. Then, we give the the proof the Proposition 1.

\subsection{Our Proposed G-invariant Neural Network Architecture}
We introduce a novel $G$-invariant neural network architecture, which exploits the theory of invariant polynomials to achieve a flexible scheme for $G$-invariant transformation of data for some known and finite group $G \leq S_{n}$, where $S_n$ is a symmetric group and $|G| = m$.

We assume that an input $x \in \mathbb{R}^{n \times n_{in}}$ to the proposed network is a tensor\footnote{We use matrix notation to denote tensors in this paper.} $x = [x_1\, x_2\,  \dots\, x_n]^T$ of $n$ vectors $x_i \in \mathbb{R}^{n_{in}}, i =1,2,\ldots, n$. 
A function $f: \mathbb{R}^{n \times n_{in}} \rightarrow \mathbb{R}$  is $G$-invariant if $f$ satisfies
\begin{equation}
\label{eq:sup_mat_Ginv}
  \allfor_{x \in \mathbb{R}^{n \times n_{in}}}\allfor_{g \in G}\,f\left(g(x)\right) = f(x), 
\end{equation}
where\footnote{$\allfor_{y \in Y} P(Y)$ means that ``predicate $P(Y)$ is true for all $y \in Y$''.} the action of the group element $g$ on $x$ is defined by
\begin{equation}
\label{eq:sup_mat_Gequiv}
    g(x) = \{x_{\sigma_g(1)}, x_{\sigma_g(2)}, \dots, x_{\sigma_g(n)}\},
\end{equation}
where $\sigma_g(i)$ denotes the action of the group element $g$ on the specific index $i$ and $x_{\sigma_g(i)} \in \mathbb{R}^{n_{in}}$.

Our proposed $G$-invariant neural network is defined by a  function ${\Gamma: \mathbb{R}^{n \times n_{in}} \rightarrow \mathbb{R}^{n_{out}}}$ of the following form
\begin{equation}
\label{eq:sup_mat_Gamma}
   \Gamma(x) = f_{out}(\Sigma\Pi(f_{in}(x))),
\end{equation}
where $f_{in}$ is a $G$-equivariant input transformation function, $\Sigma\Pi$ is a function which comprises $G$-invariant transformation when combined with $f_{in}$, and $f_{out}$ is an output transformation function.
The general idea of the proposed architecture is to define a $G$-invariant transformation, which uses the sum of $G$-invariant polynomials ($\Sigma\Pi$) of $n$ variables, which are the outputs of $f_{in}$. This transformation produces a $G$-invariant feature vector, which is processed by another function $f_{out}$ that is approximated by a Multi-Layer Perceptron.

First, let us define the $G$-equivariant input transformation function $f_{in}: \mathbb{R}^{n \times n_{in}} \rightarrow \mathbb{R}^{n \times n \times n_{mid}}$, where $n_{mid}$ is the size of the feature vector. This function can be represented as a vector ${\Phi = [\phi_1\, \phi_2\, \dots\, \phi_n]}$ of neural networks, where each function $\phi_i: \mathbb{R}^{n_{in}} \rightarrow \mathbb{R}^{n_{mid}}$ is applied on all elements of the set of input vectors $\{x_i\}_{i=1}^n$, and transforms them to the $n_{mid}$ dimensional vector.
As a result, the operation of the $f_{in}$ function can be formulated by
\begin{equation}
\label{eq:sup_mat_fin}
    f_{in}(x) = 
    \begin{bmatrix}
        \Phi(x_1)\\
        \Phi(x_2)\\
        \vdots\\
        \Phi(x_n)\\
    \end{bmatrix} = 
    \begin{bmatrix}
        \phi_1(x_1) & \ldots & \phi_n(x_1)\\
        \vdots & \ddots & \vdots \\
        \phi_1(x_n) & \ldots & \phi_n(x_n)\\
    \end{bmatrix}.
\end{equation}
One can see that $f_{in}(x)$ is $G$-equivariant, since the action of the vector $\Phi$ of functions is the same for each element of the vector $x$, thus it transposes the rows of the matrix form \eqref{eq:sup_mat_fin} according to $g \in G$, which is equivalent to transposing the rows after the calculation of $f_{in}(x)$.

Second, we define the function $\Sigma\Pi: \mathbb{R}^{n \times n \times n_{mid}} \rightarrow \mathbb{R}^{n_{mid}}$, which constructs $G$-invariant polynomials of outputs obtained from $f_{in}$, by
\vspace{-0.1in}
\begin{equation}
\label{eq:sup_mat_sigmapi}
    \Sigma\Pi(x) = \sum_{g \in G}\prod_{j=1}^{n} x_{\sigma_g(j), j}.
\end{equation}
To see the $G$-invariance of $\Sigma\Pi(f_{in}(x))$, we substitute $x$ from \eqref{eq:sup_mat_sigmapi} with \eqref{eq:sup_mat_fin} to obtain
\vspace{-0.1in}
\begin{equation}
\label{eq:sup_mat_sigmapifin_G}
    \Sigma\Pi(f_{in}(x)) = \sum_{g \in G}\prod_{j=1}^{n}  \phi_j(x_{\sigma_g(j)}).
\end{equation}
Then, we can show that \eqref{eq:sup_mat_sigmapifin_G} is $G$-invariant by checking whether \eqref{eq:sup_mat_Ginv} holds for any input $x$ and any group element $g' \in G$ as follows:
\vspace{-0.05in}
\small
\begin{equation}
\begin{split}
    \Sigma\Pi(g'(f_{in}(x))) &= \sum_{g \in G}\prod_{j=1}^{n} \phi_j(x_{\sigma_{g'}(\sigma_g(j))})\\ 
    &= \sum_{g \in G}\prod_{j=1}^{n}  \phi_j(x_{\sigma_g(j)}) = \Sigma\Pi(f_{in}(x))
\end{split},
\end{equation}
\normalsize
since any group element acting on the group leads to the group itself.
Last, we define the output function $f_{out}: \mathbb{R}^{n_{mid}} \rightarrow \mathbb{R}^{n_{out}}$ following the structure of a typical fully connected neural network by 
\vspace{-0.05in}
\begin{equation}
\label{eq:sup_mat_fout}
  f_{out}(x) = \sum_{i=1}^{N}c_i \sigma\left(\sum_{j=1}^{n_{mid}} w_{ij} x_j + h_i\right),
\end{equation}
where $N \in \mathbb{N}_{+}$ is a parameter, $\sigma$ is a non-polynomial activation function and $c_i, w_{ij}, h_i \in \mathbb{R}$ are coefficients.

\subsection{Proposition 1 and its Proof}

\begin{proposition}
\label{thm:prop}
The network function (\ref{eq:sup_mat_Gamma}), can approximate any $G$-invariant function $f: V \rightarrow \mathbb{R}$, where $V$ is a compact subset of $\mathbb{R}^{n \times n_{in}}$ and $G \leq S_n$ is a finite group, as long as number of features $n_{mid}$ at the output of input transformation network $f_{in}$ is greater than or equal to the size $N_{inv}$ of the generating set $\mathcal{F}$ of polynomial $G$-invariants.
\end{proposition}

\begin{proof}
In the proof, without the loss of generality, we consider the case when $n_{out} = 1$, as the approach can be generalized for arbitrary $n_{out}$. Moreover, we assume that 
\begin{equation}
\label{eq:sup_mat_V}
    0 \notin V
\end{equation}
to avoid the change of sign when approximating polynomials of inputs, but it is not a limitation because any compact set can be transformed to such a set by a bijective function.

To prove the Proposition \ref{thm:prop}, we need to employ two theorems:
\begin{theorem}[\cite{yarotsky}]
\label{thm:UAT}
Let $\sigma: \mathbb{R} \rightarrow \mathbb{R}$ be a continuous activation function that is not a polynomial. Let $V=\mathbb{R}^d$ be a real finite dimensional vector space.
Then, any continuous map ${f:V \rightarrow \mathbb{R}}$ can be approximated, in the sense of uniform convergence on compact sets, by
\begin{equation}
\label{eq:sup_mat_pinkus}
  \hat{f}(x_1, x_2, \dots, x_d) = \sum_{i=1}^{N}c_i \sigma\left(\sum_{j=1}^{d} w_{ij} x_j + h_i\right)
\end{equation}
with a parameter $N \in \mathbb{N}_+$ and coefficients $c_i, w_{ij}, h_{i} \in \mathbb{R}$.
\end{theorem}
\vspace{-0.15in}
The above version of the theorem comes from the work of \cite{yarotsky}, but it was proved by \cite{pinkus}.

\begin{theorem}[\cite{yarotsky}]
\label{thm:poly}
Let $\sigma: \mathbb{R} \rightarrow \mathbb{R}$ be a continuous activation function that is not a polynomial, $G$ be a compact group, $W$ be a finite-dimensional $G$-module and $f_1,\ldots,f_{N_{inv}}: W \rightarrow \mathbb{R}$ be a finite generating set of polynomial invariants on $W$ (existing by Hilbert’s theorem).
Then, any continuous invariant map $f:W \rightarrow \mathbb{R}$ can be approximated by an invariant map $\hat{f}:W \rightarrow \mathbb{R}$ of the form
\begin{equation}
\label{eq:sup_mat_yarotsky}
  \hat{f}(x) = \sum_{i=1}^{N}c_i \sigma\left(\sum_{j=1}^{N_{inv}} w_{ij} f_j(x) + h_i\right)
\end{equation}
with a parameter $N \in \mathbb{N}_+$ and coefficients $c_i, w_{ij}, h_{i} \in \mathbb{R}$.
\end{theorem}
\vspace{-0.1in}
The accuracy of the approximation (\ref{eq:sup_mat_yarotsky}) has been proven to be $2\epsilon$ for some arbitrarily small positive constant $\epsilon$. 
Note that the function $f_{out}$ is of the same form as the function $\hat{f}$. Then, one can accurately imitate the behavior of $\hat{f}$ using $f_{out}$, if the input to both functions are equivalent.

\begin{lemma}
\label{thm:lemma}
For every element $f_i: V \rightarrow \mathbb{R}$ of the finite generating set $\mathcal{F} = \{ f_i\}_{i=1} ^{N_{inv}}$ of polynomial $G$-invariants on $V$, there exists an approximation of the form (\ref{eq:sup_mat_sigmapifin_G}), linearly dependent on $\epsilon$, where $G \leq S_n$ is an $m$ element subgroup of the $n$ element permutation group and $\epsilon$ is an arbitrarily small positive constant.
\end{lemma}

\begin{proof}
Any function $f_i \in \mathcal{F}$ has the following form
\begin{equation}
\label{eq:sup_mat_inv_f_i_sigma}
    f_i(x) = \sum_{g \in G} \psi(g(x)),
\end{equation}
where
\begin{equation}
\label{eq:sup_mat_inv_f_i_pi}
    \psi(x) = \prod_{i=1}^{n} x_i^{b_i},
\end{equation}
and $b_i$ are fixed exponents.
Combining (\ref{eq:sup_mat_inv_f_i_sigma}) and (\ref{eq:sup_mat_inv_f_i_pi}), we obtain:
\begin{equation}
\label{eq:sup_mat_inv_f_i}
    f_i(x) = \sum_{g \in G} \prod_{i=1}^{n} x_{\sigma_g(i)}^{b_i},
\end{equation}
which has a similar form as \eqref{eq:sup_mat_sigmapifin_G}.
% Now, we will show that this resemblance is not accidental, but in fact, $\Sigma\Pi(f_{in}(x))$ can
This resemblance is not accidental, but in fact, $\Sigma\Pi(f_{in}(x))$ can
approximate $n_{mid}$ functions belonging to the set $\mathcal{F}$.
Using Theorem \ref{thm:UAT} and the fact that $\phi_i$ is a neural network satisfying (\ref{eq:sup_mat_pinkus}), we observe that $\phi_j(x_i)$ can approximate any continuous function with $\epsilon$ precision. Thus, it can approximate $x_{i}^{b_i}$ for some constant parameter $b_i$. It is possible to provide an upper bound on the approximation error $\left| f_i(x) - \Sigma\Pi_i(f_{in}(x)) \right|$
%where $\left| \right|$ denotes the absolute difference, 
by
\begin{equation}
\label{eq:sup_mat_sigmapifin_err}
\begin{split}
    &\left| f_i(x) - \Sigma\Pi_i(f_{in}(x)) \right| \stackrel{(\ref{eq:sup_mat_inv_f_i}, \ref{eq:sup_mat_sigmapifin_G})}{=} \\
    &\left| \sum_{g \in G} \prod_{i=1}^{n} x_{\sigma_g(i)}^{b_i} - \sum_{g \in G}\prod_{j=1}^{n}  \phi_j(x_{\sigma_g(j)}) \right| \leq\\
    & \sum_{g \in G} \left|\prod_{i=1}^{n} x_{\sigma_g(i)}^{b_i} - \prod_{j=1}^{n}  \phi_j(x_{\sigma_g(j)}) \right| \leq\\
    & \sum_{g \in G} \left|\prod_{i=1}^{n} x_{\sigma_g(i)}^{b_i} - \prod_{j=1}^{n}  (x_{\sigma_g(j)}^{b_j} - \epsilon) \right| \stackrel{(\ref{eq:sup_mat_V})}{\leq} mn\epsilon\\
\end{split},
\end{equation}
for some arbitrarily small positive constant $\epsilon$.
\end{proof}
Assuming that the number of features $n_{mid}$ at the output of input transformation network $f_{in}$ is greater than or equal to the size of the generating set $\mathcal{F}$, it is possible to estimate each of $f_i(x)$ using \eqref{eq:sup_mat_sigmapifin_G}\footnote{Size of the generating set $\mathcal{F}$ depends on the complexity of the function (in terms of approximating it with the use of polynomials). However, we conjecture that in practice the $n_{mid}$ can be lower than the size of $\mathcal{F}$, as using neural networks to transform the input features offers a much 
more flexible representation than only its powers.}.

The last step for completing the proof of the Proposition \ref{thm:prop}, using Theorem \ref{thm:UAT}, Theorem \ref{thm:poly}, and the proposed Lemma \ref{thm:lemma}, is to show that
\begin{equation}
    \left| f(x) - \Gamma(x) \right| \leq \epsilon c,
\end{equation}
where $c \in \mathbb{R}$ is a constant.

Let us consider the error
\begin{equation}
\label{eq:sup_mat_last_bound}
\begin{split}
    & \left| f(x) - \Gamma(x) \right| \stackrel{(\ref{eq:sup_mat_Gamma})}{=} \left| f(x) - \hat{f}(x)\right| + \\
    %& \left| f(x) - \hat{f}(x) + \hat{f}(x) - f_{out}(\Sigma\Pi(f_{in}(x))) \right| \leq \\
    & + \left|\hat{f}(x) - f_{out}(\Sigma\Pi(f_{in}(x))) \right| \stackrel{\text{Thm.} \ref{thm:poly}}{=} \\
    & 2\epsilon + \left|\hat{f}(x) - f_{out}(\Sigma\Pi(f_{in}(x))) \right| = \\
    %& 2\epsilon + \left|\hat{f}(x) - f_{out}(\mathcal{F}(x)) +\right.\\ 
    & \left.  f_{out}(\mathcal{F}(x)) - f_{out}(\Sigma\Pi(f_{in}(x))) \right| \leq \\
    & 2\epsilon + \left|\hat{f}(x) - f_{out}(\mathcal{F}(x)) \right| + \\
    & \left| f_{out}(\mathcal{F}(x)) - f_{out}(\Sigma\Pi(f_{in}(x))) \right| \stackrel{\text{Thm.} \ref{thm:UAT}, }{\leq} \\
    & 3\epsilon + \left| f_{out}(\mathcal{F}(x)) - f_{out}(\Sigma\Pi(f_{in}(x))) \right|
\end{split}.
\end{equation}
Several transformations presented in (\ref{eq:sup_mat_last_bound}) result in the formula which is a sum of $3\epsilon$ and the absolute difference of $f_{out}(\mathcal{F})$ and $f_{out}(\Sigma\Pi(f_{in}(x)))$. From (\ref{eq:sup_mat_sigmapifin_err}), we have that the difference of the arguments is bounded by $mn\epsilon$. Consider then a ball $B_{mn\epsilon}(x)$ with radius $mn\epsilon$ centered at $x$. Since $f_{out}$ is a MLP (multi-layer perceptron), which is at least locally Lipschitz continuous, we know that its output for $x' \in B_{mn\epsilon}(x)$ can change at most by $kmn\epsilon$, where $k$ is a Lipschitz constant. From those facts, we can provide an upper bound on the error (\ref{eq:sup_mat_last_bound}) by
\begin{equation}
\begin{split}
    & \left| f(x) - \Gamma(x) \right| = \\
    & 3\epsilon + \left| f_{out}(\mathcal{F}(x)) - f_{out}(\Sigma\Pi(f_{in}(x))) \right| \leq \\
    & 3\epsilon + kmn\epsilon = \epsilon(3 + kmn) = \epsilon c \\
\end{split}.
\end{equation}
\end{proof}

\subsection{Computational efficiency analysis}
Having proved that the proposed approach is universal we elucidate its computational and memory complexity.

The tensor with the largest size is obtained at the output of the $f_{in}$ function. The size of this tensor is equal to $n^2 n_{mid}$, where we assume that $n_{mid} \geq N_{inv}$ and it is a design parameter of the network. So, the memory complexity is of the order $n^2 n_{mid}$, which is polynomial. However, the complexity of the method proposed by \cite{maron19}, is of the order $n^p$, where $\frac{n-2}{2} \leq p \leq \frac{n(n-1)}{2}$ depending on the group $G$.

In order to evaluate the function $\Sigma\Pi$, $m (n-1) n_{mid}$ multiplications are needed, where $m = |G|$ and $n_{mid}$ is a parameter, but we should assure that $n_{mid} \geq N_{inv}$ to ensure universality of the proposed method (see Section 3.3).
It is visible, that the growth of the number of computations is linear with $m$.
For smaller subgroups of $S_n$, such as $\mathbb{Z}_n$ or $D_{2n}$, where $m \propto n$, the number of the multiplications is of order $n^2$, which is a lot better than the number of multiplications performed by the $G$-invariant neural networks proposed in \cite{maron19}, which is of order $n^p$.
However, for big groups, where $m$ approaches $n!$, the number of multiplications increases. 
Although the proposed approach can work for all subgroups of $S_n$ ($m$=$n!$), it suits the best for small and medium size groups, yet not less important, such as cyclic groups $Z_n$, $D_{2n}$, $S_k$ ($k < n$) or their direct products.

Moreover, the proposed $\Sigma\Pi$ can be implemented efficiently on GPUs using a parallel implementation of matrix multiplication and reduction operations in practice.

As in the proposed approach, we also conduct $m$ summation's through all elements of the group $G$, we would like to compare its efficiency to the group averaging approach. However, note that in case of group averaging, whole $G$-invariant processing pipeline is multiplied $m$ times, whereas for our approach term $m$ is present only in $\Sigma\Pi$ function, and the $G$-invariant pipeline is multiplied $n$ times. Thus, we can use relatively big neural networks in the first stages of the processing, as the number of computations at this stage scales linearly with $n$, not $m$. Let's consider a simple example of generation a $G$-invariant representation using two aforementioned methods with a single layer of neural network (the best case for group averaging, as its processing pipeline scales with $m$ and we reduced it to the minimum), where $n = 10$, $n_{in} = 2$, $n_{mid} = 32$, $G = S_5 \times S_5$, $m = 120^2$. Then number of multiplications for group averaging is equal to $m(n\cdot n_{in} \cdot n_{mid}) = 120^2(10\cdot2\cdot32) = 9.216 \cdot 10^6$, whereas for our approach $n \cdot n \cdot n_{in} \cdot n_{mid} + m \cdot (n - 1) \cdot n_{mid} =  10\cdot10\cdot2\cdot32 + 120^2 \cdot 9 \cdot 32 = 4.1536 \cdot 10^6$.

\end{document}